\def\year{2017}\relax

\documentclass[letterpaper]{article}
\usepackage[cp1251]{inputenc}
\usepackage[hyperindex,breaklinks]{hyperref}
\usepackage{aaai17post}
\usepackage{times}
\usepackage{helvet}
\usepackage{courier}

\usepackage{enumitem}
\usepackage{amssymb,amsmath,amsthm} 
\usepackage{stmaryrd}
\theoremstyle{definition}
\newtheorem{theorem}{Theorem}
\newtheorem{lemma}{Lemma}
\newtheorem{definition}{Definition}

\begin{document}

\pdfinfo{
/Title (Preferential Structures for Comparative Probabilistic Reasoning)
/Author (Matthew Harrison-Trainor, Wesley H. Holliday, Thomas F. Icard, III) }

\title{Preferential Structures for Comparative Probabilistic Reasoning\thanks{Postprint of AAAI 2017 paper, corrected to include a distinguished subset $W_\succeq$ in Definitions 2-3 and 5 (resp.~$W_+$ before Theorem 3) to match the semantics of Holliday and Icard 2013 (resp.~van der Hoek 1996) when $R(w)=R(v)$ for all $w,v\in W$.}}
\author{Matthew Harrison-Trainor\\
Group in Logic and the Methodology of Science\\
University of California, Berkeley\\
\And
Wesley H. Holliday\\
Department of Philosophy and\\
Group in Logic and the Methodology of Science\\
University of California, Berkeley\\
\AND
Thomas F. Icard, III\\
Department of Philosophy and\\
Symbolic Systems Program\\
Stanford University
}

\maketitle

\begin{abstract}
\vspace{-.08in}
Qualitative and quantitative approaches to reasoning about uncertainty can lead to different logical systems for formalizing such reasoning, even when the language for expressing uncertainty is the same. In the case of reasoning about relative likelihood, with statements of the form $\varphi\succsim\psi$ expressing that $\varphi$ \textit{is at least as likely as} $\psi$, a standard qualitative approach using preordered preferential structures yields a dramatically different logical system than a quantitative approach using probability measures. In fact, the standard preferential approach validates principles of reasoning that are incorrect from a probabilistic point of view. However, in this paper we show that a natural modification of the preferential approach yields exactly the same logical system as a probabilistic approach---not using single probability measures, but rather sets of probability measures. Thus, the same preferential structures used in the study of non-monotonic logics and belief revision may be used in the study of comparative probabilistic reasoning based on imprecise probabilities.
\vspace{-.08in}
\end{abstract}

\section{Introduction}

Central to being an intelligent agent is the ability to reason adequately under, and even about, uncertainty. Such reasoning requires representing this uncertainty in some way. In this paper, we study \emph{qualitative} forms of reasoning about likelihood, a subject that can be traced back to the foundations of mathematical probability theory \cite{Finetti1937}. Such reasoning may concern absolute or relative likelihood, both of which have been studied in some detail within AI. The former involves statements of the form $L\varphi$ expressing that \textit{$\varphi$ is likely} \cite{Halpern1987,Halpern1989}. The latter---our topic in this paper---involves statements of the form $\varphi\succsim\psi$ expressing that \textit{$\varphi$ is at least as likely as $\psi$}.\footnote{A different reading of $\varphi\succsim\psi$ as \textit{belief in $\varphi$ is at least as strong as belief in $\psi$} has been used to motivate another mathematical interpretation \cite{Ghosh2013}.} This representation language has been considered under several different interpretations:
\begin{itemize}
\item \textbf{Probabilistic} interpretation \cite{Segerberg1971,Gardenfors1975,Hoek1996b,Alon2014a,Delgrande2015}, building on the theory of qualitative probability \cite{Finetti1937,Narens2007};
\item \textbf{Possibilistic} interpretation \cite{Cerro1991,Dubois1998,Dubois2003,Touazi2015};
\item \textbf{Preferential} interpretation \cite{Halpern1996,Halpern1997}.
\end{itemize}
Each of these interpretations captures a distinct intuition about relative likelihood reasoning, and they may be appropriate for different settings \cite{Halpern2003}. A natural question is: to what extent are these interpretations compatible with one another? Our focus is on the relationship between the probabilistic and preferential interpretations. Existing approaches based on the preferential interpretation \cite{Halpern1997} depart rather dramatically from those based on probabilistic interpretations---in fact, relating closely to the possibilistic interpretation. We demonstrate how a natural modification of the approach---based on qualitative preferential structures---yields a calculus for reasoning about relative likelihood that is, from a logical point view, indistinguishable from standard forms of probabilistic reasoning. 

Specifically, we prove soundness and completeness for the comparative logic of imprecise probabilities \cite{Alon2014a} with respect to a class of preferential models that can be independently motivated. This shows how the two approaches can be seen as fundamentally compatible. The methods we use to prove these results further lead to a complexity result: the logic itself we prove to be NP-complete. This justifies the intuitive claim that this style of comparative probabilistic reasoning is relatively elementary, being no more complex than ordinary Boolean reasoning. 

\section{Preferential Structures}

Preferential structures have been applied widely across the field of representation and reasoning, including in the study of non-monotonic logics \cite{Shoham1987,Kraus1990}, belief revision \cite{Katsuno1991,Friedman1994}, and conditional belief in games \cite{Board2004,Baltag2008}. Restricted classes of preferential structures are also used for conditional logics \cite{Lewis1973,Friedman1994b}. Here we are interested in the application of these structures to the study of likelihood comparisons, whereby an ordering on a set of states is ``lifted'' to an ordering on the set of events over those states, i.e., on the powerset.

\section{Language}

The language $\mathcal{L}$ of relative likelihood, generated from a set \textsf{Prop} of propositional variables, is given by the grammar
\[\varphi::= p\mid \neg\varphi\mid (\varphi\wedge\varphi)\mid (\varphi\succsim\varphi),\]
where $p\in\mathsf{Prop}$. We define $\vee$, $\to$, $\leftrightarrow$, $\bot$, and $\top$ as usual. The \textit{modal depth} of a formula of $\mathcal{L}$ is defined by:
\begin{enumerate}
\item $md(p)=0$ for $p\in\mathsf{Prop}$;
\item $md(\neg\varphi)=md(\varphi)$;
\item $md(\varphi\wedge\psi)=max\{md(\varphi),md(\psi)\}$;
\item $md(\varphi\succsim\psi)=max\{md(\varphi),md(\psi)\}+1$.
\end{enumerate}
The \textit{length} $|\varphi|$ of $\varphi$ is the number of symbols in $\varphi$.

\section{Preferential Semantics}

All of the models for $\mathcal{L}$ that we will consider will add some extra structure to the following common base models.

\begin{definition}\label{StateSpace} A \textit{state space model} is a tuple $M=\langle W, V\rangle$ where $W$ is a nonempty set and $V\colon \mathsf{Prop}\to \wp(W)$. 
\end{definition}

Preferential models then add a preorder on the space $W$.

\begin{definition} A \textit{preferential model} is an $\mathcal{M}=\langle W,\succeq, V\rangle$ where $\langle W,V\rangle$ is a state space model and $\succeq$ is a preorder (reflexive and transitive binary relation) on a nonempty subset $W_\succeq$ of $W$.
\end{definition}
In a multi-agent setting, one could consider a family $\{\succeq_{w,a}\}_{w\in W,\,a\in Agents}$ of preorders on $W$, but in this paper we can make our essential points with a single preorder.

The idea of the preferential semantics for $\mathcal{L}$ is to lift the order $\succeq$ on $W$ to an order $\succeq^\uparrow$ on $\wp(W)$. A natural way to do so was studied by \cite{Halpern1996,Halpern1997}, following an earlier proposal by \cite{Lewis1973} in a different setting: 
let $A \succeq^l B$ iff for all $b \in B$ there is $a \in A$ such that $a \succeq b$. This is equivalent to there being a function $f\colon B\to A$ that is \emph{inflationary} with respect to $\succeq$, i.e., $f(x)\succeq x$ for all $x\in B$.

\begin{definition}[Inflationary function semantics]\label{def:ifs} Given a preferential model $\mathcal{M}=\langle W,\succeq,V\rangle$, state $w\in W$, and formula $\varphi\in\mathcal{L}$, we define the satisfaction relation $\mathcal{M},w\vDash_l \varphi$ by:
\begin{enumerate}
\item $\mathcal{M},w\vDash_l p$ iff $w\in V(p)$;
\item $\mathcal{M},w\vDash_l \neg\varphi$ iff $\mathcal{M},w\nvDash_l \varphi$;
\item $\mathcal{M},w\vDash_l \varphi\wedge\psi$ iff $\mathcal{M},w\vDash_l \varphi$ and $\mathcal{M},w\vDash_l \psi$;
\item $\mathcal{M},w\vDash_l\varphi\succsim\psi$ iff $\llbracket \varphi\rrbracket^\mathcal{M}\succeq^l \llbracket \psi\rrbracket^\mathcal{M}$,
\end{enumerate}
where $\llbracket \alpha\rrbracket^\mathcal{M} = \{w\in W_\succeq\mid \mathcal{M},w\vDash_l \alpha\}$.
\end{definition}
For this semantics and all those to follow, a formula is \textit{valid} over a class of models according to the semantics iff it is satisfied by every state in every model in the class.

 The motivation for studying the logic of the particular lifting operation above was in part to establish connections with default reasoning and with possibilistic models. In particular, the axiom (\textbf{L4}) below is the central principle in the axiomatization of this class of models.

\begin{definition}\label{ILlogic} The set of theorems of \textsf{IL} (the logic of \textit{inflationary lifting}) is the smallest set of formulas that contains all tautologies of propositional logic, is closed under \textit{modus ponens} (if $\varphi\in\mathsf{IL}$ and $\varphi\to\psi\in\mathsf{IL}$, then $\psi\in\mathsf{IL}$) and \textit{necessitation} (if $\varphi\in\mathsf{IL}$, then $\varphi\succsim\top\in\mathsf{IL}$), and contains all instances of the following axiom schemas:
\begin{enumerate}
\item[] \textbf{(L1)} $\varphi \succsim \varphi$
\item[] \textbf{(L2)} $\big(\bot \succsim (\psi \wedge\neg \varphi) \big) \rightarrow (\varphi \succsim \psi)$
\item[] \textbf{(L3)} $\big((\varphi \succsim \psi) \wedge (\psi \succsim \chi)\big) \rightarrow (\varphi \succsim \chi)$
\item[] \textbf{(L4)} $\big((\varphi \succsim \psi) \wedge (\varphi \succsim \chi)\big) \rightarrow \big(\varphi \succsim (\psi \vee \chi)\big)$
\item[] \textbf{(I1)} $(\varphi\succsim\psi)\to ((\varphi\succsim\psi)\succsim \top)$
\item[] \textbf{(I2)} $\neg(\varphi\succsim\psi)\to (\neg(\varphi\succsim\psi)\succsim \top)$.
\end{enumerate}
\end{definition}

\begin{theorem}[Halpern 2003] The logic \textsf{IL} is sound and complete with respect to preferential models according to the inflationary lifting semantics. \label{halpern}
\end{theorem}

While the principle (\textbf{L4}) may be sensible in some contexts, it is patently incompatible with probabilistic reasoning, where taking the union of many improbable events can lead to probable events. A natural question is whether there is an alternative lifting operation that might accord better with probabilistic reasoning. Such a lifting was proposed in \cite{HollidayIcard2013}: instead of requiring an inflationary function $B$ to $A$, let us say that $A \succeq^i B$ iff there is an inflationary \emph{injection} from $B$ to $A$. Call this the \emph{inflationary injection semantics} for the language $\mathcal{L}$. 

\begin{definition}[Inflationary injection semantics]\label{InjectionSemantics} Given a preferential model $\mathcal{M}=\langle W,\succeq,V\rangle$, state $w\in W$, and formula $\varphi\in\mathcal{L}$, we define $\mathcal{M},w\vDash \varphi$ in the same way as Definition \ref{def:ifs}, except with clause 4 replaced by:
\begin{enumerate}
\item[4.] $\mathcal{M},w\vDash\varphi\succsim\psi$ iff $\llbracket \varphi\rrbracket^\mathcal{M}\succeq^i \llbracket \psi\rrbracket^\mathcal{M}$,
\end{enumerate}
where $\llbracket \alpha\rrbracket^\mathcal{M} = \{w\in W_\succeq\mid \mathcal{M},w\vDash \alpha\}$.
\end{definition}

The intuition behind this alternative is clear. If we are trying to determine whether event $A$ is at least as likely as event $B$, we might try to find, for each possible $B$ state, an $A$ state that is as least as likely as it. However, we cannot match multiple $B$ states with the same $A$ state, since together these $B$ states may become more likely than the $A$ state. The inflationary injection semantics generalizes this intuition to require each $B$ state to correspond with its own unique $A$ state that weakly dominates it in likelihood.

We would now like to verify that this intuition does indeed lead to a system for reasoning about relative likelihood that validates intuitive principles. In particular, we would like a completeness theorem along the lines of Theorem \ref{halpern}. In that direction we first introduce a previously studied logic for reasoning about probabilistic semantics for language $\mathcal{L}$.

\section{Probabilistic Semantics}

Probabilistic models replace the preorder $\succeq$ by a set $\mathcal{P}$ of probability measures. A singleton set $\mathcal{P}$ represents \textit{sharp probabilities}, while multiple disagreeing measures in $\mathcal{P}$ can be used to represent \textit{imprecise probabilities}.

\begin{definition} A \textit{multi-measure model} is a tuple $\mathfrak{M}=\langle W,\mathcal{P},V\rangle$ where $\langle W,V\rangle$ is a state space model and $\mathcal{P}$ is a set of finitely additive measures $\mu\colon \wp(W)\to [0,1]$.

A \textit{single-measure model} is a multi-measure model in which $|\mathcal{P}|=1$. If $\mathcal{P}=\{\mu\}$, we write $\mathfrak{M}=\langle W, \mu, V\rangle$.
\end{definition}

Likelihood comparisons $\varphi\succsim\psi$ are now decided by a unanimity rule over $\mathcal{P}$ as follows.

\begin{definition}[Multi-measure semantics]\label{MultiSem} Given a multi-measure model $\mathfrak{M}=\langle W,\mathcal{P},V\rangle$, state $w\in W$, and formula $\varphi\in\mathcal{L}$, we define $\mathfrak{M},w\vDash \varphi$ with the usual clauses for propositional variables and Boolean connectives, plus:
\[\mathfrak{M},w\vDash\varphi\succsim\psi\mbox{ iff }\forall \mu\in\mathcal{P}\colon \mu(\llbracket \varphi\rrbracket^\mathfrak{M})\geq \mu(\llbracket \psi\rrbracket^\mathfrak{M})\]
where $\llbracket \varphi\rrbracket^\mathfrak{M}=\{w\in W\mid \mathfrak{M},w\vDash\varphi\}$.
\end{definition}

The next two lemmas will be useful for our later results. The first reduces formulas to a convenient form.

\begin{lemma}\label{depth0} In each of the semantics (multi-measure, inflationary function, inflationary injection) above, every formula $\varphi\in \mathcal{L}$ is equivalent to a formula $\psi$ of modal depth $\leq 1$. Moreover, we may take $\psi$ to be a disjunction in which the length of each disjunct is $O(|\varphi|)$.
\end{lemma}

\begin{proof} Suppose $md(\varphi)>1$, which implies there is an inequality $\alpha\succsim\beta$ inside the scope of another inequality, where $md(\alpha)=md(\beta)=0$. Define $\varphi':=\varphi'_1\vee \varphi'_2$ by \[\varphi':= (\varphi[\top / \alpha\succsim\beta] \wedge \alpha\succsim\beta)\vee  (\varphi[\bot / \alpha\succsim\beta] \wedge \neg( \alpha\succsim\beta)).\]
For any model $\mathbb{M}$ with domain $W$, either $\llbracket \alpha\succsim\beta\rrbracket^\mathbb{M}=W$ or $\llbracket \alpha\succsim\beta\rrbracket^\mathbb{M}=\varnothing$, which implies that $\llbracket \varphi\rrbracket^\mathbb{M}=\llbracket\varphi'\rrbracket^\mathbb{M}$.  The number of inequalities in the scope of other inequalities in $\varphi'_i$ is less than that in $\varphi$, and the length of $\varphi_i'$ is greater than that of $\varphi$ by a constant amount.  Thus, repeating this process on $\varphi'_1$ and $\varphi'_2$, we eventually obtain the desired $\psi$.\end{proof}

The next lemma bounds the number of measures we need.

\begin{lemma}\label{FinitelyMany} If there is a multi-measure model $\langle W, \mathcal{P},V\rangle$ and $w\in W$ that satisfies  $\varphi$, then there is a $\mathcal{P}_0\subseteq\mathcal{P}$ whose size is $O(|\varphi|)$ such that $w$ satisfies $\varphi$ in $\langle W, \mathcal{P}_0,V\rangle$.
\end{lemma}

\begin{proof} By Lemma \ref{depth0} and propositional logic, $\varphi$ is equivalent to a disjunction of conjunctions of the form 
\[ \neg (\varphi_1 \succsim \psi_1) \wedge \cdots \wedge \neg (\varphi_n \succsim \psi_n) \wedge \xi\]
where $\xi$ is a conjunction of modal depth 0 formulas and inequalities $\varphi \succsim \psi$ between such formulas. It suffices to satisfy one of the disjuncts, and for this it suffices to have one measure $\mu_i$ for each of the $n$ negated inequalities. Let $\mathcal{P}_0=\{\mu_1,\dots,\mu_n\}$. Since the satisfaction of the positive inequalities in $\xi$ is preserved under discarding measures from $\mathcal{P}$, we have that $\varphi$ is still satisfied by $w$ in $\langle W,\mathcal{P}_0,V\rangle$.
\end{proof}

\section{Logical Systems $\mathsf{SP}$ and $\mathsf{IP}$}

For logics of probability models, we first define the logic of single-measure models. To do so, it will be helpful to introduce some abbreviations from \cite{Segerberg1971,Gardenfors1975}. Given formulas $\varphi_1,\dots,\varphi_n$, $\psi_1,\dots,\psi_n$ and $0\leq k\leq n$, define $C_k$ to be the disjunction of all conjunctions
\[f_1\varphi_1\wedge\dots\wedge f_n\varphi_n\wedge g_1\psi_1\wedge\dots\wedge g_n\psi_n \]
where exactly $k$ of the $f$'s and $k$ of the $g$'s are the empty string, and the rest are $\neg$. Thus, $C_k$ is true at a state $w$ iff exactly $k$ of the $\varphi$'s and $k$ of the $\psi$'s are true at $w$. Then let
\[(\varphi_{1},\ldots,\varphi_{n}) \equiv(\psi_{1},\ldots,\psi_{n}):= C_0\vee\dots\vee C_n,\]
which is true at a state $w$ iff the number of $\varphi$'s true at $w$ is exactly the same as the number of $\psi$'s true at $w$.

\begin{definition} The set of theorems of \textsf{SP} (the logic of \textit{sharp probability}) is the smallest set of formulas that contains all tautologies of propositional logic, is closed under \textit{modus ponens} and \textit{necessitation} as in Definition \ref{ILlogic}, and contains  all instances of \textbf{(I1)}, \textbf{(I2)}, and the following for all $n\in\mathbb{N}$:\footnote{We use the labeling of axioms in \cite{Alon2014a}.}
\begin{enumerate}[leftmargin=0in]
\item[] \textbf{(A0)} $(\varphi\succsim\psi)\vee(\psi\succsim\varphi)$
\item[] \textbf{(A1)} $\varphi\succsim \bot$\quad\textbf{(A2)} $\varphi\succsim\varphi$\quad\textbf{(A3)} $\neg(\bot\succsim\top)$
\item[]\textbf{(A4)}  $\big[ (\varphi_1\succsim\psi_1)\wedge \dots\wedge (\varphi_n\succsim\psi_n) \;\wedge$ \\
$(\varphi_{1},\ldots,\varphi_{n},\varphi') \equiv(\psi_{1},\ldots,\psi_{n},\psi')\succsim\top \big] \to (\psi'\succsim\varphi')$.
\end{enumerate}
\end{definition}

It is easy to see that each axiom of $\mathsf{SP}$ is valid over single-measure models. In the case of \textbf{(A4)},  if the set of states in $\mathfrak{M}$ that make exactly the same number of formulas from $(\varphi_1,\dots,\varphi_n,\varphi')$ true as from $(\psi_1,\dots,\psi_n,\psi')$ has measure 1, and  $\mu(\llbracket \varphi_i\rrbracket^\mathfrak{M})\geq\mu(\llbracket \psi_i\rrbracket^\mathfrak{M})$ for $1\leq i\leq n$, then it is impossible to have $\mu(\llbracket \varphi'\rrbracket^\mathfrak{M})>\mu(\llbracket \psi'\rrbracket^\mathfrak{M})$. The soundness of $\mathsf{SP}$ is thus straightforward, but the completeness of $\mathsf{SP}$ is far less so. The completeness proof in \cite{Segerberg1971,Gardenfors1975} uses a famous representation theorem of \cite{Kraft1959,Scott1964}.

\begin{theorem}[Segerberg 1971, G\"{a}rdenfors 1975] $\mathsf{SP}$ is sound and complete with respect to the class of single-measure models.
\end{theorem}

As shown by van der Hoek \cite{Hoek1996b}, $\mathsf{SP}$ can be given a simple semantics using what we will call \textit{distinguished state models} $M=(W,W_+,V)$ where $(W,V)$ is a state space model and $\varnothing\neq W_+\subseteq W$. Then we define
\[M,w\vDash \varphi\succsim \psi\mbox{ iff }| \llbracket \varphi \rrbracket^M | \geq | \llbracket \psi\rrbracket^M®  |,\]
where $\llbracket \varphi \rrbracket^M=\{w\in W_+\mid M,w\vDash \varphi\}$ and $|\cdot|$ is cardinality.

\begin{theorem}[van der Hoek 1996] $\mathsf{SP}$ is sound and complete with respect to the class of finite distinguished state models with the cardinality semantics above.
\end{theorem}

The logic $\mathsf{SP}$ is not sound, however, with respect to all \textit{multi}-measure models, as axiom $(\textbf{A0})$ is clearly invalid; for one measure $\mu\in\mathcal{P}$ we may have $\mu(\llbracket \varphi\rrbracket^\mathfrak{M})>\mu(\llbracket \psi\rrbracket^\mathfrak{M})$ while for another measure $\mu'\in\mathcal{P}$ we may have $\mu'(\llbracket \psi\rrbracket^\mathfrak{M})>\mu'(\llbracket \varphi\rrbracket^\mathfrak{M})$, in which case neither $\varphi\succsim\psi$ nor $\psi\succsim\varphi$ is true in $\mathfrak{M}$ according to Definition \ref{MultiSem}. Thus, for multi-measure models we must drop the axiom $(\textbf{A0})$; but if we only drop $(\textbf{A0})$ from $\mathsf{SP}$, then the resulting logic will be sound but not \textit{complete} with respect to multi-measure models \cite{HHI2016b}. We must not only drop $(\textbf{A0})$ but also strengthen $(\textbf{A4})$ as follows.

\begin{definition} The set of theorems of $\mathsf{IP}$ (the logic of \textit{imprecise probability}) is defined in the same way as $\mathsf{SP}$ except without axiom \textbf{(A0)} and with \textbf{(A4)} replaced by:
\begin{itemize}[leftmargin=0in]
\item[]\textbf{(A4')}  $\big[ (\varphi_1\succsim\psi_1)\wedge \dots\wedge (\varphi_n\succsim\psi_n) \;\wedge$
\begin{eqnarray*}
(\varphi_{1},\ldots,\varphi_{n},\underbrace{\varphi',\ldots,\varphi'}_{k\text{ times}}) \equiv(\psi_{1},\ldots,\psi_{n},\underbrace{\psi',\ldots,\psi'}_{k\text{ times}})\succsim\top \big] \\
\to (\psi'\succsim\varphi').
\end{eqnarray*}
\end{itemize}
\end{definition}

Again soundness is not difficult, while completeness relies on a representation theorem in \cite{Alon2014}.

\begin{theorem}[Alon and Heifetz 2014]  $\mathsf{IP}$ is sound and complete with respect to the class of all multi-measure models.
\end{theorem}

\section{Main Result}

Our main result is to prove that the logic one obtains from finite preferential models with the inflationary injection semantics is exactly the logic $\mathsf{IP}$ of imprecise probability that one obtains from multi-measure models.

\begin{theorem} $\mathsf{IP}$ is sound and complete with respect to the class of finite preferential models with the inflationary injection semantics. Moreover, it is complete with respect to the class of finite preferential models in which $\succeq$ is a \textit{total} preorder on $W_\succeq$ (i.e., for all $w,v\in W_\succeq$, $w\succeq v$ or $v\succeq w$).\end{theorem}

We will first prove soundness in the next section and then completeness in the following section.

\section{Soundness of $\mathsf{IP}$}\label{sec-sound}

In proving the soundness of \textsf{IP}, the main task is show that the axiom \textbf{(A4')} is valid according to the inflationary injection semantics. To do so, we will use a result about inflationary injections from \cite{HHI2016}. This result involves the following relation between sequences of sets, which matches the relation expressed by the second line in the display of axiom \textbf{(A4')} above.

\begin{definition}
Let $S$ be a finite set. For any two sequences $\langle E_1,\dots,E_k\rangle$ and $\langle F_1,\dots,F_k\rangle$ of events from $\mathcal{P}(S) $, 
\[\langle E_1,\dots,E_k\rangle=_0\langle F_1,\dots,F_k\rangle\]
if and only if for all $s\in S$, the cardinality of $\{i\mid s\in E_i\}$ is equal to the cardinality of $\{i\mid s\in F_i\}$.
\end{definition} 

If $\langle E_1,\dots,E_k\rangle=_0\langle F_1,\dots,F_k\rangle$, then we say that the two sequences are \textit{balanced}; every state appears the same number of times on the left side as on the right side. 

The following lemma is easily seen to follow from Lemma 3.6 and the proof of Lemma 3.7 in \cite{HHI2016}.

\begin{lemma}[Harrison-Trainor, Holliday, and Icard 2016]\label{mainlem1}
Let $\langle W,\succeq\rangle$ be a finite preorder. If \[\langle E_1,\dots, E_n,\underbrace{A,\dots,A}_{r\,\mathrm{ times}}\rangle =_0 \langle F_1,\dots , F_n,\underbrace{B,\dots,B}_{r\,\mathrm{ times}}\rangle\] are balanced sequences of subsets of $W$, and there is an inflationary injection $f_i \colon F_i \to E_i$ for all $i$, then there is an inflationary injection $g\colon A \rightarrow B$.
\end{lemma}

We will now prove our soundness theorem, using Lemma \ref{mainlem1} to verify axiom \textbf{(A4')}.\footnote{Lemma \ref{mainlem1} and Theorem \ref{SoundThm} hold more generally for models in which there is no infinite sequence $x_0,x_1,x_2,\dots$ of distinct states with $x_{n+1}\succeq x_n$ \cite{HHI2016}.}

\begin{theorem}\label{SoundThm} \textsf{IP} is sound with respect to finite preferential models according to the inflationary injection semantics.
\end{theorem}
\begin{proof}
Obviously all tautologies are valid, and the rules of modus ponens and necessitation preserve validity. It only remains to show that \textbf{(A1)}--\textbf{(A3)}, \textbf{(A4')}, and \textbf{(I1)}--\textbf{(I2)} hold at every state $w$ in each preferential model $\mathcal{M} = \langle W,\succeq,V\rangle$.

\textbf{(A1)} holds via the trivial injection $\varnothing \to \llbracket \varphi \rrbracket^{\mathcal{M}}$. \textbf{(A2)} holds via the identity injection $\llbracket \varphi \rrbracket^{\mathcal{M}} \to \llbracket \varphi \rrbracket^{\mathcal{M}}$. \textbf{(A3)} holds because there is no injection $W \to \varnothing$.

For \textbf{(A4')}, given
\begin{align*}
&\mathcal{M},w \vDash(\varphi_1 \succsim \psi_1) \wedge \dots \wedge (\varphi_n \succsim \psi_n) \;\wedge \\
&(\varphi_{1},\ldots,\varphi_{n},\underbrace{\varphi',\ldots,\varphi'}_{k\text{ times}}) \equiv(\psi_{1},\ldots,\psi_{n},\underbrace{\psi',\ldots,\psi'}_{k\text{ times}}) \succsim \top,
\end{align*}
we have to show that $\mathcal{M},w \vDash\psi' \succsim \varphi'$.

We have inflationary injections $f_i \colon \llbracket \psi_i \rrbracket^{\mathcal{M}} \to \llbracket \varphi_i \rrbracket^{\mathcal{M}}$. There is also an inflationary injection
\begin{align*}
&W \to \\
&\llbracket (\varphi_{1},\ldots,\varphi_{n},\underbrace{\varphi',\ldots,\varphi'}_{k\text{ times}}) \equiv(\psi_{1},\ldots,\psi_{n},\underbrace{\psi',\ldots,\psi'}_{k\text{ times}}) \rrbracket^{\mathcal{M}}.
\end{align*}
Since $W$ is finite, it follows that
\begin{align*}
&W=\\
&\llbracket (\varphi_{1},\ldots,\varphi_{n},\underbrace{\varphi',\ldots,\varphi'}_{k\text{ times}}) \equiv(\psi_{1},\ldots,\psi_{n},\underbrace{\psi',\ldots,\psi'}_{k\text{ times}}) \rrbracket^{\mathcal{M}},
\end{align*}
which implies
\begin{align*}
&\langle \llbracket\varphi_{1}\rrbracket^{\mathcal{M}},\ldots,\llbracket\varphi_{n}\rrbracket^{\mathcal{M}},\underbrace{\llbracket\varphi'\rrbracket^{\mathcal{M}},\ldots,\llbracket\varphi'\rrbracket^{\mathcal{M}}}_{k\text{ times}} \rangle \\
=_0 &\langle \llbracket\psi_{1}\rrbracket^{\mathcal{M}},\ldots,\llbracket\psi_{n}\rrbracket^{\mathcal{M}},\underbrace{\llbracket\psi'\rrbracket^{\mathcal{M}},\ldots,\llbracket\psi'\rrbracket^{\mathcal{M}}}_{k\text{ times}} \rangle.
\end{align*}
Then by Lemma \ref{mainlem1}, there is an inflationary injection $\llbracket \varphi' \rrbracket^{\mathcal{M}} \to \llbracket \psi' \rrbracket^{\mathcal{M}}$, and so $\mathcal{M},w \vDash\psi' \succsim \varphi'$.

For \textbf{(I1)}, suppose that $\mathcal{M},w' \vDash\varphi \succsim \psi$, say via an inflationary injection $f \colon \llbracket \psi \rrbracket^{\mathcal{M}} \to \llbracket \varphi \rrbracket^{\mathcal{M}}$. Then in fact, for all $w$, $\mathcal{M},w \vDash\varphi \succsim \psi$ via $f$. Hence $\llbracket \varphi \succsim \psi \rrbracket^{\mathcal{M}} = W$, and so $\mathcal{M},w \vDash(\varphi \succsim \psi) \succsim \top$. For \textbf{(I2)}, by a similar argument, if $\mathcal{M},w \vDash\neg (\varphi \succsim \psi)$, then for all $w'$, $\mathcal{M},w' \vDash\neg (\varphi \succsim \psi)$, so that $\llbracket \neg (\varphi \succsim \psi) \rrbracket^{\mathcal{M}} = W$, and $\mathcal{M},w \vDash \neg(\varphi \succsim \psi) \succsim \top$.
\end{proof}

\section{Completeness of $\mathsf{IP}$}\label{sec-comp}

Our method of proving completeness is to show that any finite multi-measure model can be transformed into a finite preferential model that satisfies the same formulas according to the inflationary injection semantics. We will then invoke the completeness of \textsf{IP} with respect to finite multi-measure models, as proved by \cite{Alon2014a}, to establish the completeness of \textsf{IP} with respect to finite preferential models with the inflationary injection semantics.

We begin with a restatement of Lemma 3.2.2 in \cite{Hoek1996b}. 

\begin{lemma}\label{Hoek} For any finite single-measure model $\mathfrak{M}=\langle W,\mu,V\rangle$, there is a distinguished state model $M^\#=\langle W^\#,W^\#_+, V^\#\rangle$ with $W\subseteq W^\#$ such that for all $w\in W$ and $\varphi\in\mathcal{L}$: 
\[\mathfrak{M},w\vDash \varphi\mbox{ iff }M^\#,w\vDash \varphi.\]
\end{lemma}

The transformation first rationalizes the probability measure, then normalizes the measure so that it takes on only integer values (but does not necessarily assign measure one to the whole state space), and finally duplicates each point in the state space that has nonzero measure with a number of points equal to its measure, which become the points in~$W^\#_+$.

Using Lemma \ref{Hoek}, we can transform multi-measure models into preferential models as follows.

\begin{lemma}\label{bridge} For any multi-measure model $\mathfrak{M}=\langle W,\mathcal{P},V\rangle$ with $W$ and $\mathcal{P}$ finite and $w\in W$, there is a finite preferential model $\mathcal{M}$ and state $v$ such that for all $\varphi\in\mathcal{L}$:
\[\mbox{$\mathfrak{M},w\vDash \varphi$ iff $\mathcal{M},v\vDash \varphi$}.\]
Moreover, we may take $\succeq$ in $\mathcal{M}$ to be total on its field $W_\succeq$.
\end{lemma}

\begin{proof} Where $\mathcal{P}=\{\mu_i\}_{i\in I}$,  define for each $i\in I$ a single-measure model $\mathfrak{M}_i=\langle W_i,\mu_i,V_i\rangle$ where $W_i=W\times \{i\}$ and $V_i(p)= V(p)\times \{i\}$. By Lemma \ref{Hoek}, for each $\mathfrak{M}_i$ there is a distinguished state model $M_i^\#=\langle W_i^\#,W^\#_{i+},V_i^\#\rangle$ with $W_i\subseteq W_i^\#$ such that for all $v\in W_i$ and $\varphi\in\mathcal{L}$: 
\begin{equation}
\mathfrak{M}_i,v\vDash_\mu \varphi\mbox{ iff }M_i^\#,v\vDash \varphi.\label{vh1}
\end{equation} 
Without loss of generality, assume that the domains $W_i^\#$ are pairwise disjoint. Now define the preferential model $\mathcal{M}=\langle W',\succeq, V'\rangle$ as follows:
\begin{itemize}
\item[] (a) $W'=\underset{i\in I}{\bigcup} W_i^\#$;
\item[] (b) $V'(p)=\underset{i\in I}{\bigcup} V_i^\#(p)$;
\item[] (c) $w\succeq v$ iff  for some $i\in I$, we have $w,v\in W^\#_{i+}$.
\end{itemize}
First, note that for any formula $\varphi$ of modal depth $0$:
\begin{equation}
\llbracket \varphi\rrbracket^\mathcal{M}=\underset{i\in I}{\bigcup}\llbracket\varphi\rrbracket^{M_i^\#}.\label{pf}
\end{equation}
Next, we claim that for any $i\in I$ and formula $\varphi$:
\begin{equation}\mathfrak{M},w\vDash \varphi\mbox{ iff }\mathcal{M},\langle w,i\rangle\vDash \varphi.\label{goal}\end{equation} 
By Lemma \ref{depth0}, it suffices to consider formulas of modal depth $\leq 1$. The proof of (\ref{goal}) is by induction with obvious atomic and Boolean cases.  For the modal case, consider $\varphi\succsim\psi$ where $\varphi$ and $\psi$ are of modal depth $0$. If $\mathfrak{M},w\vDash \varphi\succsim \psi$, then for all $i\in I$, $\mathfrak{M}_i,\langle w,i\rangle\vDash \varphi\succsim\psi$ and hence $M_i^\#,\langle w,i\rangle\vDash \varphi\succsim\psi$ by (\ref{vh1}). Since $M_i^\#,\langle w,i\rangle\vDash \varphi\succsim\psi$, there is an injection $f_i\colon \llbracket \psi\rrbracket^{M_i^\#}\rightarrow \llbracket \varphi\rrbracket^{M_i^\#}$. Then $g=\underset{i\in I}{\bigcup} f_i$ is an injection;  $g$ is inflationary by (c) above; and $g\colon \llbracket \psi\rrbracket^{\mathcal{M}}\rightarrow \llbracket \varphi\rrbracket^{\mathcal{M}}$ by (\ref{pf}). Thus, $\mathcal{M},\langle w,i\rangle\vDash \varphi\succsim\psi$. 

Conversely, if $\mathfrak{M},w\nvDash \varphi\succsim \psi$, then there is an $i\in I$ with $\mathfrak{M}_i,\langle w,i\rangle\nvDash \varphi\succsim\psi$ and hence $M_i^\#,\langle w,i\rangle\nvDash \varphi\succsim\psi$ by (\ref{vh1}). Thus, there is no injection $f_i\colon \llbracket \psi\rrbracket^{M_i^\#}\rightarrow \llbracket \varphi\rrbracket^{M_i^\#}$. Now suppose for reductio that there is an inflationary injection $g\colon \llbracket \psi\rrbracket^{\mathcal{M}}\rightarrow \llbracket \varphi\rrbracket^{\mathcal{M}}$. Given (c), that $g$ is an \textit{inflationary} injection implies that  $g\upharpoonright W_{i+}^\#$ is an injection from $\llbracket \psi\rrbracket^{M_i^\#}$ to $\llbracket \varphi\rrbracket^{M_i^\#}$, contradicting the above. Thus, there is no such inflationary injection $g$, so $\mathcal{M},\langle w,i\rangle\nvDash \varphi\succsim\psi$.

Finally, we will show how to transform $\mathcal{M}$ into a  model $\mathcal{M}^t$ where $\succeq^t$ is total on $W_\succeq^t$. Just as the domain of $\mathcal{M}$ is the union of disjoint domains $W_i^\#$ for $i\in I$, the domain of $\mathcal{M}^t$ will be the union of disjoint domains $W_i^t$. Suppose the index set $I$ with which we began is $I=\{0,\dots,m\}$. We define the domains $W_i^t$ inductively as follows, for some $d\not\in W'$:
\begin{itemize}
\item[](d$_1$) $W_0^t=W_0^\# \times \{d\}$;
\item[](d$_2$) $W_{n+1}^t=W_{n+1}^\# \times ( \underset{i\leq n}\bigcup W_i^t \cup\{d\}) $;
\item[](d$_3$) $W^t = \underset{i\in I}{\bigcup}W_i^t$.
\end{itemize}
We multiply $W_0^\#$ by $\{d\}$ so that every state in $W^t$ will be a pair, the first coordinate of which is a state from $W'$. Then $W_1^t$ is the result of making $|W_0^t|+1$-many copies of each state in $W_1^\#$; $W_2^t$ is the result of making $|W_0^t\cup W_1^t|+1$-many copies of each state in $W_2^\#$; and so on. 

We define the valuation $V^t$ by
\begin{itemize}
\item[](e) $V^t(p)= \{\langle w,v\rangle\in W^t\mid w\in V'(p)\}$,
\end{itemize}
so each copy of a state in $\mathcal{M}^t$ has the same propositional valuation that the original state had in $\mathcal{M}$. 

Where $W^t_{n+}=\{\langle w,v\rangle\in W^t_n\mid w\in W^\#_{n+}\}$, define
\begin{itemize}
\item[](f) $x\succeq^t y$ iff for some $k\leq \ell$, $x\in W^t_{k+}$ and $y\in W^t_{\ell+}$.
\end{itemize}
Then clearly $\succeq^t$ is a total preorder on its field.

We claim that for any $\langle w,v\rangle\in W^t$ and $\varphi\in\mathcal{L}$:
\begin{equation}
\mathcal{M}^t,\langle w,v\rangle\vDash \varphi\mbox{ iff }\mathcal{M},w\vDash\varphi.
\end{equation}
The proof is by induction on $\varphi$ with the atomic case given by the definition of $V^t$. The Boolean cases are also routine. 

For the modal case, as before we may consider $\varphi\succsim\psi$ where $\varphi$ and $\psi$ are of modal depth 0. If $\mathcal{M},w\vDash \varphi\succsim\psi$, then in $\mathcal{M}$ there is a $\succeq$-inflationary injection $f\colon \llbracket \psi\rrbracket^\mathcal{M}\to\llbracket\varphi\rrbracket^\mathcal{M}$. By the inductive hypothesis, we have 
\begin{eqnarray}
\llbracket \psi\rrbracket^{\mathcal{M}^t}&=&\{\langle x,z\rangle\in W^t\mid x\in \llbracket \psi\rrbracket^\mathcal{M}\}\label{NewExtension1}\\
\llbracket \varphi\rrbracket^{\mathcal{M}^t}&=&\{\langle y,z\rangle\in W^t\mid y\in \llbracket \varphi\rrbracket^\mathcal{M}\} \label{NewExtension2}.
\end{eqnarray}
Since $f$ is $\succeq$-inflationary, $f(x)=y$ implies that $x$ and $y$ belong to the same $W_{i+}^\#$ by (c) above. Thus, by (d$_1$)-(d$_3$), we have that $x$ and $y$ get copied by the same elements in $W^t$: $\langle x,z\rangle\in W^t$ iff $\langle y,z\rangle\in W^t$. Then given (\ref{NewExtension1})-(\ref{NewExtension2}), it follows that the function $f^t$ sending each $\langle x,z\rangle\in \llbracket\psi\rrbracket^{\mathcal{M}^t}$ to $\langle f(x),z\rangle \in \llbracket\varphi\rrbracket^{\mathcal{M}^t}$ is a $\succeq^t$-inflationary injection.

Now suppose $\mathcal{M},w\nvDash \varphi \succsim\psi$, so there is no $\succeq$-inflationary injection from $ \llbracket \psi\rrbracket^\mathcal{M}$ to $\llbracket\varphi\rrbracket^\mathcal{M}$. By construction of $\mathcal{M}$, it follows that for some $i\in I$, there is no injection from  $ \llbracket \psi\rrbracket^\mathcal{M}\cap W_i^\#$ to $ \llbracket \varphi\rrbracket^\mathcal{M}\cap W_i^\#$. Thus, 
\[ |\llbracket \psi\rrbracket^\mathcal{M}\cap W_i^\#| >  |\llbracket \varphi\rrbracket^\mathcal{M}\cap W_i^\#|.\] 
It follows by (d$_1$)-(d$_3$) and (\ref{NewExtension1})-(\ref{NewExtension2}) that
\begin{equation}
|\llbracket \psi\rrbracket^{\mathcal{M}^t}\cap W_i^t| > |\llbracket \varphi\rrbracket^{\mathcal{M}^t}\cap W_i^t| + | \underset{j< i}{\bigcup} W_{j}^t |.  \label{card}
\end{equation}
Now by the definition of $\succeq^t$ in (f), any $\succeq^t$-inflationary injection from $ \llbracket \psi\rrbracket^{\mathcal{M}^t}\cap W_i^t$ to $ \llbracket \varphi\rrbracket^{\mathcal{M}^t}$ must be an injection from $ \llbracket \psi\rrbracket^{\mathcal{M}^t}\cap W_i^t$ to $ \llbracket \varphi\rrbracket^{\mathcal{M}^t} \cap \underset{j\leq i}{\bigcup}W_j^t$, but such an injection cannot exist by (\ref{card}). Thus, there is no $\succeq^t$-inflationary injection from $ \llbracket \psi\rrbracket^{\mathcal{M}^t}$ to $ \llbracket \varphi\rrbracket^{\mathcal{M}^t}$, so $\mathcal{M}^t, \langle w,v\rangle\nvDash \varphi\succsim\psi$.
\end{proof}

Using Lemma \ref{bridge}, we now exploit the known completeness result for multi-measure models to obtain the following.

\begin{theorem} \textsf{IP} is complete with respect to finite preferential models (with $\succeq$ total on $W_\succeq$) according to the inflationary injection semantics.
\end{theorem}

\begin{proof} If $\varphi$ is not a theorem of \textsf{IP}, then by Theorem 2 of \cite{Alon2014a}, there is a finite multi-measure model that satisfies $\neg\varphi$. We may assume it has only finitely many measures by Lemma \ref{FinitelyMany}. Thus, by Lemma \ref{bridge}, there is a finite preferential model (with $\succeq$ total on $W_\succeq$) satisfying $\neg\varphi$ according to the inflationary injection semantics.
\end{proof}

\section{Complexity of $\mathsf{IP}$}

We now consider the problem of deciding whether a formula is consistent in $\mathsf{IP}$ (a formula $\varphi$ being consistent if $\neg\varphi\not\in\mathsf{IP}$). We will show that this problem is NP-complete. Our strategy will be to show that if a formula $\varphi$ is satisfiable, then it is satisfiable in a multi-measure model of polynomially bounded size, in which the measure of each state is a rational number of polynomially bounded size. A similar argument can be given using preferential models, but for the sake of space we will use a known result for measure models.

In particular, we will use a result from \cite{Fagin1990}, where it was shown that the satisfiability problem, for single-measure models, of a logic allowing for rational comparisons of probability is NP-complete. The language they considered is essentially an extension of the formulas of our language of modal depth $\leq 1$ to allow for rational comparisons. Thus, the following is immediate from Theorem 2.6 of \cite{Fagin1990}.

\begin{theorem}\label{Complexity1}
Suppose $\varphi$ is a formula of modal depth $\leq 1$ that is satisfied in some finite single-measure model. Then $\varphi$ is satisfied in a single-measure model $\langle W,\mu,V\rangle$ with  $|W|\leq |\varphi|$ and where the probability assigned to each state is a rational number with size $O(|\varphi | \log |\varphi|)$.
\end{theorem}

By the size of a rational number, we mean a bound on the binary representations of its numerator and denominator. 

From Theorem \ref{Complexity1}, we obtain an analogous result for multi-measure models using the strategy of Lemma \ref{bridge}.

\begin{theorem}\label{SmallModels}
Suppose $\theta$ is a formula of any modal depth that is satisfied in some finite multi-measure model. Then $\theta$ is satisfied in a multi-measure model $\langle W,\mathcal{P},V\rangle$ with $|W| \leq O(|\theta|^2)$, $|\mathcal{P}| \leq O(|\theta|)$, and where the probabilities assigned to each state are rational numbers with size $O(|\theta | \log |\theta |)$.
\end{theorem}
\begin{proof}
Given a formula $\theta$, applying Lemma \ref{depth0} and putting each disjunct in disjunctive normal form, as in Lemma \ref{FinitelyMany} we can rewrite $\theta$ as a disjunction of formulas
\[ \neg (\varphi_1 \succsim \psi_1) \wedge \cdots \wedge \neg (\varphi_n \succsim \psi_n) \wedge \xi\]
of modal depth $\leq 1$, where $\xi$ is a conjunction of modal depth 0 formulas and inequalities $\varphi \succsim \psi$ between modal depth 0 formulas. By Lemma \ref{depth0}, the length of each disjunct is $O(|\theta|)$. Since $\theta$ is satisfiable iff one of the disjuncts is, it suffices to show that if a disjunct  $\varphi$ is satisfiable then it is satisfiable in a model of bounded size as in the statement of the theorem. 

If the disjunct
\[\varphi:= \neg (\varphi_1 \succsim \psi_1) \wedge \cdots \wedge \neg (\varphi_n \succsim \psi_n) \wedge \xi\]
is satisfiable in a finite multi-measure model $\mathfrak{M} = \langle W,\mathcal{P},V\rangle $ at a state $w$, then for each $i\leq n$, there is a measure $\mu_i$ such that $\mu_i (\llbracket \varphi_i \rrbracket^{\mathfrak{M}}) < \mu_i (\llbracket \psi_i \rrbracket^{\mathfrak{M}})$. Then $\langle W,\mu_i,V\rangle,w \vDash\neg (\varphi_i \succsim \psi_i) \wedge \xi$. So each formula $\neg (\varphi_i \succsim \psi_i) \wedge \xi$ is satisfiable in a finite single-measure model. Thus, by Theorem \ref{Complexity1}, each  $\neg (\varphi_i \succsim \psi_i) \wedge \xi$ is satisfiable at $w_i$ in a finite single-measure model $\mathfrak{M}_i^* = \langle W_i^*,\mu_i^*,V_i^*\rangle$ with at most $O(|\varphi|)$ states and where the probability assigned to each state is a rational number with size $O(|\varphi| \log |\varphi|)$.

Let $W^*$ be the disjoint union of the $W_i^*$. We can extend each $\mu_i^*$ to all of $W^*$ by having $\mu_i^*$ assign measure zero to $W^* - W_i^*$. Let $\mathcal{P}^*$ be the set of the $\mu_i^*$, and let $V^*$ be the valuation on $W^*$ such that $V^*(p)$ is the union of the $V_i^*(p)$'s. Let $\mathfrak{M}^* = (W^*,\mathcal{P}^*,V^*)$. Then for any $w = w_i$,
\[ \mathfrak{M}^*,w \vDash\neg (\varphi_1 \succsim \psi_1) \wedge \cdots \wedge \neg (\varphi_n \succsim \psi_n) \wedge \xi.\]
As $|W^*| \leq O(|\varphi|^2)$ and $|\mathcal{P}^*| \leq O(|\varphi|)$, we are done.
\end{proof}

From Theorem \ref{SmallModels}, the complexity result easily follows.

\begin{theorem}
The problem of deciding whether a formula is consistent in the logic $\mathsf{IP}$ is NP-complete.
\end{theorem}
\begin{proof}
The problem is clearly NP-hard, as it generalizes the satisfiability problem for propositional logic. It is not hard to see that it is in NP; a certificate which witnesses that a formula $\theta$ is satisfiable is the small model from Theorem \ref{SmallModels}, which is of polynomially bounded size in $|\theta|$; and we can check in polynomial time that $\theta$ is true in this model.
\end{proof}

\section{Discussion}

As we have shown, the preferential approach to reasoning about relative likelihood is quite compatible with a probabilistic approach. Beginning with a preorder on states, we can reason about event comparisons in a way that coheres perfectly with quantitative probabilistic reasoning.

These results are of interest not only for the field of representation and reasoning within AI, but also in other domains such as theoretical linguistics. In the area of natural language semantics, for example, what we have called the \emph{inflationary function} semantics was independently proposed as a model of epistemic comparatives, e.g., ``at least as likely as'' in English \cite{Kratzer1991}. While this is evidently the dominant approach in the field, several authors have argued that the validity of principle (\textbf{L4}) disqualifies it as a model of ordinary speakers' intuitions about relative likelihood. The specific question has been raised as to whether there is an alternative lifting operation---that is, a different preferential approach---whose logic fits better with ordinary intuitions \cite{Yalcin2010}. Insofar as \textsf{IP}, the logic of sets of measures, matches intuitions, our results answer this question.

More generally, the present work can be seen as part of a broader project to explore possible ways of unifying logical and probabilistic approaches to AI \cite{Russell}. In addition to combining logical and probabilistic tools, another important strand of this project is to clarify when well-understood qualitative tools---such as preferential structures---and familiar forms of probabilistic reasoning---such as reasoning about sets of measures---can already be seen as two sides of the same logical coin. 

\bibliographystyle{aaai}
\bibliography{Preferential}
            
\end{document}